\documentclass[letterpaper, 10 pt, conference]{ieeeconf}  %
\IEEEoverridecommandlockouts                              %
\usepackage{amsthm} 
\usepackage{amssymb}  
\usepackage{breqn}
\newtheorem{problem}{Problem}
\newtheorem{theorem}{Theorem}
\newtheorem{corollary}{Corollary}
\newtheorem{lemma}{Lemma}
\newtheorem{remark}{Remark}

\newtheorem{assumption}{Assumption}
\usepackage{amsmath,amsfonts,amssymb,color}
\usepackage{mathtools}

\usepackage{epsfig}
\usepackage{algorithm}
\usepackage{algpseudocode}
\usepackage{dsfont}
\usepackage{multirow}

\usepackage{array}
\usepackage{multirow}
\usepackage{mathtools}

\DeclareMathOperator*{\argmin}{\arg\!\min}
\usepackage{epstopdf}
\usepackage{tikz}
\usepackage{relsize}
\usetikzlibrary{shapes,arrows}
\usepackage{cite}
\usepackage{epstopdf}
\usepackage{tikz}
\usetikzlibrary{shapes}
\definecolor{winered}{rgb}{0.5,0,0}
\usepackage{scalerel}
\usepackage{xcolor}
\usepackage[colorlinks]{hyperref}
\AtBeginDocument{%
  \hypersetup{
    citecolor=blue,
    linkcolor=winered}}

\title{\LARGE \bf
Distributed Statistical Min-Max Learning in the Presence of \\ Byzantine Agents  
}

\author{Arman Adibi\footnotemark*\thanks{*Arman Adibi and Aritra Mitra contributed equally.}, Aritra Mitra\footnotemark*, George J. Pappas and Hamed Hassani
\thanks{The authors are with the Department of Electrical and Systems Engineering, University of Pennsylvania. Email: {\tt \{aadibi, amitra20, hassani, pappasg\}@seas.upenn.edu}. This work was supported by NSF CPS Grant 1837253, ARL CRA DCIST, NSF CAREER award CIF 1943064, and the Air Force Office
of Scientific Research Young Investigator Program (AFOSR-YIP) under award FA9550-20-1-0111.}%
}

\begin{document}

\maketitle
\thispagestyle{empty}
\pagestyle{empty}

\begin{abstract}
Recent years have witnessed a growing interest in the topic of min-max optimization, owing to its relevance in the context of generative adversarial networks (GANs), robust control and optimization, and reinforcement learning. Motivated by this line of work, we consider a multi-agent min-max learning problem, and focus on the emerging challenge of contending with worst-case Byzantine adversarial agents in such a setup. By drawing on recent results from robust statistics, we design a robust distributed variant of the extra-gradient algorithm - a popular algorithmic approach for min-max optimization. Our main contribution is to provide a crisp analysis of the proposed robust extra-gradient algorithm for smooth convex-concave and smooth strongly convex-strongly concave functions. Specifically, we establish statistical rates of convergence to approximate saddle points. Our rates are near-optimal, and reveal both the effect of adversarial corruption and the benefit of collaboration among the non-faulty agents. Notably, this is the first paper to provide formal theoretical guarantees for large-scale distributed min-max learning in the presence of adversarial agents. 
\end{abstract}

\section{Introduction}
We consider a min-max learning problem of the form
\begin{equation}
    \min_{x\in\mathcal{X}} \max_{y\in\mathcal{Y}} f(x,y) \triangleq \mathbb{E}_{\xi \sim \mathcal{D}}[F(x,y;\xi)].
\label{eqn:objective}
\end{equation}
Here, $\mathcal{X}$ and $\mathcal{Y}$ are convex, compact sets in $\mathbb{R}^{n}$ and $\mathbb{R}^{m}$, respectively; $x\in\mathcal{X}$ and $y\in\mathcal{Y}$ are model parameters; $\xi$ is a random variable representing a data point sampled from the distribution $\mathcal{D}$; and $f(x,y)$ is the population function corresponding to the stochastic function $F(x,y; \xi)$. Throughout this paper, we assume that $f(x,y)$ is continuously differentiable in $x$ and $y$, and is \textit{convex-concave} over $\mathcal{X} \times \mathcal{Y}$. Specifically, $f(\cdot, y): \mathcal{X} \rightarrow \mathbb{R}$ is convex for every $y\in\mathcal{Y}$, and $f(x,\cdot): \mathcal{Y} \rightarrow \mathbb{R}$ is concave for every $x \in \mathcal{X}$. Our goal is to find a saddle point $(x^*,y^*)$ of $f(x,y)$ over the set $\mathcal{X} \times \mathcal{Y}$, where a saddle point is defined as a vector pair $(x^*,y^*) \in \mathcal{X} \times \mathcal{Y}$ that satisfies
\begin{equation}
    f(x^*,y) \leq f(x^*,y^*) \leq f(x,y^*), \forall x \in \mathcal{X}, y \in \mathcal{Y}.
\label{eqn:saddle_point}
\end{equation}

The min-max optimization problem described above features in a variety of applications: from classical developments in game theory \cite{von} and online learning \cite{cesa}, to robust optimization \cite{ben} and reinforcement learning \cite{dai}. More recently, in the context of machine learning,  min-max problems have found important applications in training generative adversarial networks (GANs) \cite{goodfellow,arjovsky}, and in robustifying deep neural networks against adversarial attacks \cite{madry}. Motivated by this recent line of work, we consider a min-max learning problem of the form in Eq. \eqref{eqn:objective}, where the data samples required for finding a saddle-point are distributed across multiple devices (agents). Specifically, we focus on a large-scale distributed setup comprising of $M$ agents, each of which can access i.i.d. data samples from the distribution $\mathcal{D}$. The agents collaborate under the orchestration of a central server to compute an approximate saddle point of statistical accuracy higher relative to the setting when they act alone. The intuition here is simple: since all agents receive data samples from the \textit{same} distribution, exchanging information via the server can help reduce the randomness (variance) associated with these samples.\footnote{This intuition will be made precise in Section \ref{sec:Results}.} An example of the above setup that aligns with the modern federated learning  paradigm is one where  multiple devices (e.g., cell phones or tablets) collaborate via a server to train a robust statistical model; see, for instance, \cite{farnia}. 

To reap the benefits of collaboration in modern distributed computing systems, one needs to contend with the critical challenge of \textit{security}. In particular, this challenge arises from the fact that the individual agents in such systems are easily susceptible to adversarial attacks. In fact, unless appropriately accounted for, even a single malicious agent can severely degrade the overall performance of the system by sending corrupted messages to the central server. 

\textbf{Objective.} Thus, given the emerging need for security in large-scale computing, \textit{our objective in this paper is to design an algorithm that achieves  near-optimal statistical performance in the context of distributed min-max learning, while being robust to worst-case attacks}. To that end, we consider a setting where a fraction of the agents is Byzantine \cite{lamport}. Each Byzantine agent is assumed to have complete knowledge of the system and learning algorithms; moreover, leveraging such knowledge, the Byzantine agents can send arbitrary messages to the server and collude with each other. 

\textbf{Challenges.} Even in the absence of noise or attacks, recent work \cite{daskalakis} has shown that algorithms such as gradient descent ascent (\texttt{GDA}) can diverge for simple convex-concave functions. We have to contend with both noise (due to our statistical setup) \textit{and} worst-case attacks - this makes the analysis for our setting non-trivial. In particular, the adversarial agents can introduce complex probabilistic dependencies across iterations that need to be carefully accounted for; we do so in this work by making the following contributions. 

\textbf{Contributions.} Our contributions are summarized below.

$\bullet$ \textit{Problem.} Given the importance and relevance of security, several recent works have studied distributed optimization/learning in the face of adversarial agents. However, we are unaware of any analogous paper for adverarially-robust distributed \textit{min-max} learning. Our work closes this gap.

$\bullet$ \textit{Algorithm.} In Section \ref{sec:Algo}, we develop an algorithm for finding an approximate saddle point to the min-max learning  problem in Eq.~\eqref{eqn:objective}, subject to the presence of Byzantine agents. Our proposed algorithm - called Robust Distributed Extra-Gradient (\texttt{RDEG}) - brings together two separate algorithmic ideas: (i) the classical extra-gradient algorithm due to Korpelevich \cite{korpelevich} that has gained a lot of popularity due to its empirical performance in training GANs, and (ii) the recently proposed univariate trimmed mean estimator due to Lugosi and Mendelson \cite{lugosi}. 

$\bullet$ \textit{Theoretical Results.} Our main contribution is to provide a rigorous theoretical analysis of the performance of \texttt{RDEG} for smooth convex-concave (Theorem \ref{thm:CC}), and smooth strongly convex-strongly concave (Theorem \ref{thm:SC}) settings. In each case, we establish that as long as the fraction of corrupted agents is ``small", \texttt{RDEG} guarantees convergence to approximate saddle points at \textit{near-optimal} statistical rates with high probability. The rates that we derive precisely highlight the benefit of collaboration in effectively reducing the variance of the noise model. At the same time, they indicate the (unavoidable) additive bias introduced by adversarial corruption. Notably, our results in the context of min-max learning complement those of a similar flavor in \cite{yin} for stochastic optimization under attacks. However, our analysis differs significantly from that in \cite{yin}: unlike the covering argument employed in \cite{yin}, our proofs rely on a simpler, and more direct probabilistic analysis. An immediate benefit of such an analysis is that one can build on it for the more challenging nonconvex-nonconcave setting as future work. 

\textbf{Related Work.} In what follows, we discuss connections to relevant strands of literature.

$\bullet$ \textit{Min-Max Optimization.} Convergence guarantees of first-order algorithms for saddle point problems over compact sets were studied in \cite{nemirov} and \cite{nedic}. More recently, there has been a surge of interest in analyzing the performance of such algorithms from different perspectives: a dynamical systems approach in \cite{liang,daskalakis2}, and a proximal point perspective in \cite{aryan}. We refer to \cite{linsurv} for a detailed survey on this topic. 

$\bullet$ \textit{Robust Distributed Optimization and Learning.} Robustness to adversarial agents in distributed optimization has been extensively studied in \cite{su_vaidya,sundaram,ravi}. However, these works consider deterministic settings, and do not provide statistical error rates like we do. In the context of statistical learning over a server-client computing architecture, several works have proposed and analyzed robust algorithms \cite{blanchard,yin,chen,chensu,pillutla}. Notably, none of the above works consider the min-max learning problem studied in this paper. 

$\bullet$ \textit{Robust Statistics.} Robust mean estimation in the presence of outliers is a classical topic in statistics pioneered by Huber \cite{huber,huber2}, with follow-up work in \cite{minsker,cheng}. In our work, we exploit some recent results on this topic from \cite{lugosi}. 

\section{Problem Formulation}
\label{sec:prob}
\begin{figure}[t]
\centering
  \includegraphics[width=0.95\linewidth]{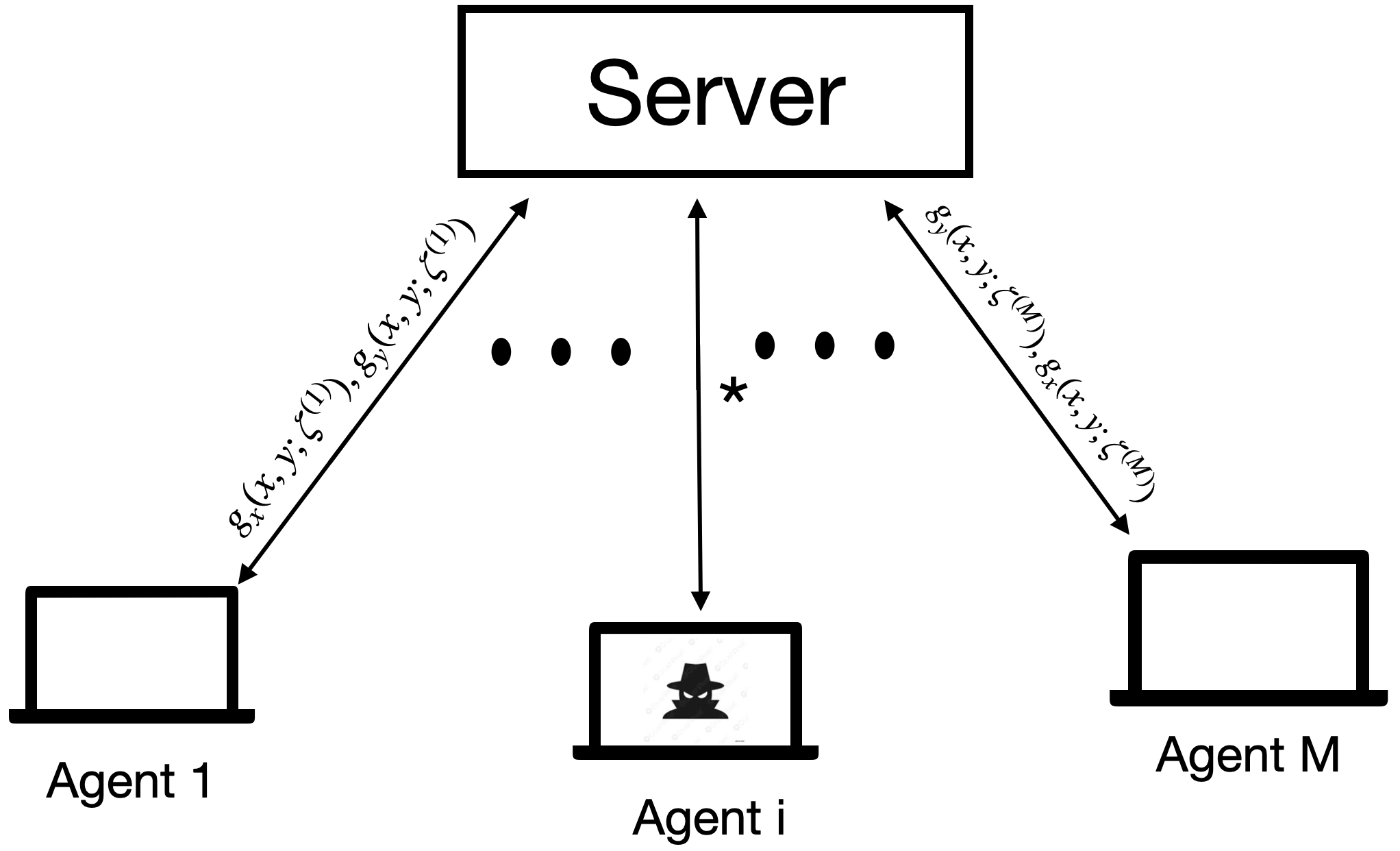}
  \caption{A group of $M$ agents collaborate to find a saddle point for the min-max learning problem in Eq.~\eqref{eqn:objective}. A fraction $\alpha$ of the agents is adversarial and upload arbitrarily corrupted messages (denoted by $*$) to the server. All the remaining good agents upload noisy partial gradients of $f(x,y)$.}
  \label{fig:model}
\end{figure} 
In this section, we formally set up the problem of interest by first introducing some notation. Our setting comprises of $M$ agents, $\alpha M$ of whom are Byzantine; see Fig.~\ref{fig:model}. We denote the adversarial agents by $\mathcal{B}\in [M]$.\footnote{Given a positive integer $N$, we use $[N]$ to represent the set $\{1,\ldots, N\}$.} For any $\bar{x} \in \mathcal{X}$ and $\bar{y}\in\mathcal{Y}$, let  $\nabla_x f(\bar{x},\bar{y})$ and $\nabla_y f(\bar{x},\bar{y})$ denote the gradient of $f(x,y)$ with respect to $x$ and $y$, respectively, at $(\bar{x},\bar{y})$. Upon drawing a sample $\xi \sim \mathcal{D}$ at a point $(\bar{x},\bar{y})$, each normal agent receives noisy estimates of $\nabla_x f(\bar{x},\bar{y})$ and $\nabla_y f(\bar{x},\bar{y})$ denoted by $g_x(\bar{x},\bar{y};  \xi)$ and $g_y(\bar{x},\bar{y}; \xi)$, respectively. For each normal agent in $[M]\setminus\mathcal{B}$, these noisy estimates satisfy the following for all $\bar{x} \in \mathcal{X}$ and $\bar{y} \in \mathcal{Y}$:
\begin{equation}
\begin{aligned}
\mathbb{E}_{\xi \sim \mathcal{D}} [g_x(\bar{x},\bar{y}; \xi)]&=\nabla_x f(\bar{x},\bar{y}) \\ \mathbb{E}_{\xi \sim \mathcal{D}}[g_y(\bar{x},\bar{y}; \xi)]&=\nabla_y f(\bar{x},\bar{y}). 
\end{aligned}
\label{eqn:stoch1}
\end{equation}
Furthermore, $\forall j \in [n]$ and $\forall k \in [m]$, we have 
\begin{equation}
\begin{aligned}
\mathbb{E}_{\xi \sim \mathcal{D}}\left[{\Vert [g_x(\bar{x},\bar{y};\xi)]_j - [\nabla_x f(\bar{x},\bar{y})]_j \Vert}^2\right] &\leq \sigma_x^2(j) \\ 
\mathbb{E}_{\xi \sim \mathcal{D}}\left[{\Vert [g_y(\bar{x},\bar{y};\xi)]_k - [\nabla_y f(\bar{x},\bar{y})]_k \Vert}^2\right] &\leq \sigma_y^2(k),
\end{aligned}
\label{eqn:stoch2}
\end{equation}
where we used $[a]_j$ to represent the $j$-th component of a vector $a$.\footnote{We use $\Vert \cdot \Vert$ to represent the Euclidean norm.} In words, each normal agent receives unbiased estimates of the gradients of $f(x,y)$ (w.r.t. $x$ and $y$) with component-wise bounded variance - essentially, a standard stochastic oracle model. With a slight abuse of notation, we will continue to use $\{g_x({x},{y}; \xi), g_y({x},{y}; \xi)\}$ to denote the gradients transmitted by an adversarial agent as well; these could, however, be arbitrary corrupted vectors. Our problem of interest can now be stated as follows.

\begin{problem}
\label{prob:prob}
Given access to the stochastic oracle model described by equations \eqref{eqn:stoch1} and \eqref{eqn:stoch2}, design a distributed algorithm that finds a saddle point (in the sense of Eq.~\eqref{eqn:saddle_point}) for the function $f(x,y)$ in Eq.~\eqref{eqn:objective}, despite the presence of the Byzantine adversarial set $\mathcal{B}$.  
\end{problem}

In the next section, we will develop our proposed algorithm to address Problem \ref{prob:prob}. 
\section{Robust Distributed Extra-Gradient}
\label{sec:Algo}

\begin{algorithm}[t]
\caption{Robust Distributed Extra-Gradient  (\texttt{RDEG})} 
\label{algo:RDEG}
\begin{algorithmic}[1]
\Require Initial vectors $x_1 \in \mathcal{X}$, $y_1 \in \mathcal{Y}$; algorithm parameters: step-size $\eta > 0$ and trimming parameter $\epsilon$.
\For {$t=1,\ldots,T$} 
\State Server sends $(x_t, y_t)$ to each agent.
\State Each \textit{normal} agent $i$ draws an i.i.d. sample $\xi^{(i)}_{1,t}  \sim \mathcal{D}$, and transmits $g_x({x}_t,{y}_t;\xi^{(i)}_{1,t})$,  $g_y({x}_t,{y}_t;\xi^{(i)}_{1,t})$ to  server.\footnotemark
\State Server computes robust gradients: 
\begin{equation}
\begin{aligned}
    \tilde{g}_x(x_t,y_t) &\gets \texttt{Trim}_{\epsilon}\{g_x({x}_t,{y}_t; \xi^{(i)}_{1,t}) : i \in [M]\} \\
    \tilde{g}_y(x_t,y_t) &\gets \texttt{Trim}_{\epsilon}\{g_y({x}_t,{y}_t; \xi^{(i)}_{1,t}) : i \in [M]\}.
\end{aligned}
\label{eqn:rob_grad1}
\end{equation}
\State Server computes mid-points $(\hat{x}_t, \hat{y}_t)$ as follows, and transmits them to each agent.
\begin{equation}
\begin{aligned}
    \hat{x}_t &\gets \Pi_{\mathcal{X}}\left(x_t-\eta \tilde{g}_x(x_t,y_t) \right) \\
    \hat{y}_t &\gets \Pi_{\mathcal{Y}}\left(y_t+\eta \tilde{g}_y(x_t,y_t) \right).
\end{aligned}
\label{eqn:mid_point}
\end{equation}
\State Each \textit{normal} agent $i$ draws an i.i.d. sample $\xi^{(i)}_{2,t}  \sim \mathcal{D}$, and transmits $g_x(\hat{x}_t,\hat{y}_t;\xi^{(i)}_{2,t})$,  $g_y(\hat{x}_t,\hat{y}_t;\xi^{(i)}_{2,t})$ to server.
\State Server computes robust gradients: 
\begin{equation}
\begin{aligned}
    \tilde{g}_x(\hat{x}_t,\hat{y}_t) &\gets \texttt{Trim}_{\epsilon}\{g_x(\hat{x}_t,\hat{y}_t; \xi^{(i)}_{2,t}) : i \in [M]\} \\
    \tilde{g}_y(\hat{x}_t,\hat{y}_t) &\gets \texttt{Trim}_{\epsilon}\{g_y(\hat{x}_t,\hat{y}_t; \xi^{(i)}_{2,t}) : i \in [M]\}.
\end{aligned}
\label{eqn:rob_grad2}
\end{equation}
\State Server computes new updates $x_{t+1}$ and $y_{t+1}$:
\begin{equation}
\begin{aligned}
    {x}_{t+1} &\gets \Pi_{\mathcal{X}}\left(x_t-\eta \tilde{g}_x(\hat{x}_t,\hat{y}_t)\right) \\
    {y}_{t+1} &\gets \Pi_{\mathcal{Y}}\left(y_t+\eta \tilde{g}_x(\hat{x}_t,\hat{y}_t)\right).
\end{aligned}
\label{eqn:end_point}
\end{equation}
\EndFor
\end{algorithmic}
\end{algorithm} 
\footnotetext{Recall that $\{g_x({x}_t,{y}_t; \xi^{(i)}_{1,t}), g_y({x}_t,{y}_t; \xi^{(i)}_{1,t})\}$ could be arbitrary vectors for an adversarial agent $i\in \mathcal{B}$.}

In this section, we develop the Robust Distributed Extra-Gradient (\texttt{RDEG}) algorithm outlined in Algorithm \ref{algo:RDEG}. Our algorithm evolves in discrete-time iterations $t\in [T]$, where $T$ is the total number of iterations. 
 There are two main steps in \texttt{RDEG}. In the first step, the server computes robust gradient estimates $\{\tilde{g}_x(x_t,y_t), \tilde{g}_y(x_t,y_t)\}$ at the current iterate $(x_t, y_t)$  by applying a \texttt{Trim} operator to the gradients collected from all agents (line 4); we will describe this operator shortly.  The robust gradient estimates are then used to compute a mid-point $(\hat{x}_t,\hat{y}_t)$ by performing a projected primal-dual update (line 5). In the second step, the server now computes robust gradients at the mid-point (line 7), and performs a projected primal-dual update using these gradients to generate the next iterate $(x_{t+1},y_{t+1})$. We now describe the \texttt{Trim} operation.

The \texttt{Trim} operator in equations \eqref{eqn:rob_grad1} and \eqref{eqn:rob_grad2} takes as input $M$ vectors, and applies the univariate trimmed mean estimator in \cite{lugosi} - described in Algorithm \ref{alg:trim} - to each coordinate of these vectors separately. To describe the trimmed mean estimator, suppose the data comprises of $M$ independent copies of a scalar random variable $Z$ with mean $\mu_Z$ and variance $\sigma^2_Z$. An adversary corrupts at most $\alpha M$ of these copies; the corrupted data-set is then made available to the estimator. The estimator splits the corrupted data set into two equal chunks, denoted by $Z_1, \ldots, Z_{M/2}$, $\tilde{Z}_1, \ldots \tilde{Z}_{M/2}$. One of the chunks is used to compute appropriate quantile levels for truncation (line 2 of Algo. \ref{alg:trim}). The robust estimate $\hat{\mu}_Z$ of $\mu_Z$ is an average of the data points in the other chunk, with those data points falling outside the estimated quantile levels truncated prior to averaging (line 3 of Algo. \ref{alg:trim}). 

\begin{algorithm}
\caption{Univariate Trimmed-Mean Estimator \cite{lugosi}}\label{alg:trim}
\begin{algorithmic}[1]
\Require Corrupted data set $Z_1, \ldots, Z_{M/2}$, $\tilde{Z}_1, \ldots \tilde{Z}_{M/2}$, corruption fraction $\alpha$, and confidence level $\delta$. 
\State Set $\epsilon = 8 \alpha + 24\frac{\log(4/\delta)}{M}$.
\State Let $Z^*_1 \leq Z^*_2 \leq \cdots \leq  Z^*_{M/2}$ represent a non-decreasing arrangement of $\{Z_i\}_{i\in [M/2]}$. Compute quantiles: $\gamma = Z^*_{\epsilon M/2}$ and $\beta=Z^*_{(1-\epsilon) M/2}$.  
\State Compute robust mean estimate $\hat{\mu}_Z$ as follows:
\begin{equation*}
   \hat{\mu}_Z = \frac{2}{M}\sum_{i = 1}^{M/2} \phi_{\gamma, \beta}(\tilde{Z}_i);  
\phi_{\gamma, \beta}(x) = \begin{cases}
\beta & x>\beta\\
x & x \in [\gamma, \beta] \\ 
\gamma & x < \gamma
\end{cases} 
\end{equation*}
\end{algorithmic}
\end{algorithm}

The following result on the performance of Algorithm \ref{alg:trim} will play a key role in our subsequent analysis of \texttt{RDEG}. 

\begin{theorem} \cite[Theorem 1]{lugosi} 
\label{thm:lugosi}
Consider the trimmed mean estimator in Algorithm \ref{alg:trim}. Suppose $\alpha \in [0,1/16)$, and let $\delta \in (0,1)$ be such that $\delta \geq 4 e^{-M/2}$. Then, there exists an universal constant $c$, such that with probability at least $1-\delta$,
$$ |\hat{\mu}_Z - \mu_Z| \leq c \sigma_Z\left(\sqrt{\alpha}+\sqrt{\frac{\log(1/\delta)}{M}} \right). $$
\end{theorem}

In the next section, we will provide rigorous guarantees on the performance of our proposed algorithm \texttt{RDEG}. 
\section{Performance Guarantees for \texttt{RDEG}}
\label{sec:Results}
Before stating our main results, we first make a standard smoothness assumption on the function $f(x,y)$.

\begin{assumption}
\label{ass:smoothness} There exists a constant $L > 0$ such that the following holds for all $x_1, x_2 \in \mathcal{X}$, and all $y_1, y_2 \in \mathcal{Y}$:
\begin{equation}
\begin{aligned}
    \Vert \nabla_x f(x_1,y_1) - \nabla_x f(x_2,y_2) \Vert &\leq L \left(\Vert x_1-x_2 \Vert + \Vert y_1 - y_2 \Vert \right), \\
    \Vert \nabla_y f(x_1,y_1) - \nabla_y f(x_2,y_2) \Vert &\leq L \left(\Vert x_1-x_2 \Vert + \Vert y_1 - y_2 \Vert \right).
\end{aligned}   
\nonumber
\end{equation}
\end{assumption}
We now define a few key quantities that will show up in our main results. Let $\sigma_x=\sqrt{\sum_{j\in[n]}\sigma_x^2(j)}$, $\sigma_y=\sqrt{\sum_{k\in[m]}\sigma_y^2(k)}$, and $\sigma=\max\{\sigma_x,\sigma_y\}$. Moreover, let $d=\max\{n,m\}$, and $D=\max\{D_x,D_y\}$, where $D_x$ and $D_y$ are the diameters of the sets $\mathcal{X}$ and $\mathcal{Y}$, respectively. With the above notations in place, we state our first main result that provides a bound on the primal-dual gap $\phi_T \triangleq \max_{y\in \mathcal{Y}} f(\bar{x}_T,y)-\min_{x\in\mathcal{X}}f(x,\bar{y}_T)$, where 
$$ \bar{x}_T=(1/T)\sum_{t\in[T]} \hat{x}_t, \hspace{1mm} \text{and} \hspace{1mm} \bar{y}_T=(1/T)\sum_{t\in[T]} \hat{y}_t.$$

\begin{theorem}
\label{thm:CC} Suppose Assumption \ref{ass:smoothness} holds, the fraction $\alpha$ of corrupted devices satisfies $\alpha \in [0,1/16)$, and the number of agents $M$ is sufficiently large: $M \geq 48 \log(16dT^2)$. Then, with a step-size $\eta$ satisfying  $\eta \leq 1/(2L)$, and the confidence parameter $\delta$ in Algorithm \ref{alg:trim} set to $\delta=1/(4dT^2)$, \texttt{RDEG} guarantees the following with probability at least $1-1/T$:
\begin{equation}
    \phi_T \leq \frac{D^2}{\eta T} + \tilde{O}\left(\sigma D \left(\sqrt{\alpha}+\sqrt{\frac{1}{M}}\right)\right).
\label{eqn:CC_dual_bnd} 
\end{equation}
\end{theorem}
Proofs of all our results are deferred to Section \ref{sec:Proofs}.\footnote{In the statement of our results, we will use the $\tilde{O}(\cdot)$ notation to hide terms that are logarithmic in $n,m$, and $T$.}

\textbf{Discussion.} Theorem \ref{thm:CC} tells us that with high probability, the primal-dual gap $\phi_T$ converges to a ball of radius $\tilde{O}\left(\sigma D \left(\sqrt{\alpha}+\sqrt{{1}/{M}}\right)\right)$ at a $O(1/T)$ rate. Notably, the primal-dual gap is zero if and only if $(\bar{x}_T,\bar{y}_T)$ is a saddle point of $f(x,y)$ over the set $\mathcal{X} \times \mathcal{Y}$. Thus, \texttt{RDEG} provably generates approximate saddle points. The following result is one of the main implications of Theorem \ref{thm:CC}.

\begin{corollary}
\label{corr:avg_it}
Suppose the conditions in Theorem \ref{thm:CC} hold. Then, \texttt{RDEG} guarantees the following with probability at least $1-1/T$:
\begin{equation}
    |f(\bar{x}_T,\bar{y}_T)-f(x^*, y^*)| \leq \frac{D^2}{\eta T} + \tilde{O}\left(\sigma D \left(\sqrt{\alpha}+\sqrt{\frac{1}{M}}\right)\right).
\label{eqn:CC_finalbnd}
\end{equation}
\end{corollary}

Corollary \ref{corr:avg_it} tells us that with high probability, 
the function values $f(\hat{x}_t,\hat{y}_t)$ of the averaged iterates generated by \texttt{RDEG} converge to the saddle-point value $f(x^*,y^*)$ up to an error-floor of  $\tilde{O}\left(\sigma D \left(\sqrt{\alpha}+\sqrt{{1}/{M}}\right)\right)$, at a $O(1/T)$ rate. There are several key messages from this result. First, in the absence of adversaries (i.e., when $\alpha=0$), the classical extra-gradient algorithm with a constant step-size would yield  convergence to the saddle-point value with an error floor of $\tilde{O}(\sigma(\sqrt{{1}/{M}}))$ at a $O(1/T)$ rate. Thus, modulo the biasing effect of the adversaries, the statistical performance of \texttt{RDEG} is \textit{near-optimal}. Second, the additive biasing effect due to adversarial corruption shows up even in the context of stochastic minimization \cite{yin}. In fact, the authors in \cite{yin} argue that an additive biasing effect of order  $\tilde{\Omega}(\alpha)$ is unavoidable, albeit for the minimization setting. This is all to say that the dependence of our rate on the corruption level in Eq. \eqref{eqn:CC_finalbnd} is only to be expected. Third, when the corruption level is small, the benefit of collaboration is evident from the second term in Eq. \eqref{eqn:CC_finalbnd}: the variance $\sigma$ arising from the noise term is effectively reduced by a factor of $\sqrt{M}$ due to the averaging effect of the normal agents. This effect will be aptly demonstrated by the simulations in Section \ref{sec:sims}. 

We now turn to the goal of achieving faster convergence rates than those in Theorem \ref{thm:CC}. To that end, we study the performance of the \texttt{RDEG} algorithm for strongly convex-strongly concave (SC-SC) functions. Accordingly, we first make the following assumption on $f(x,y)$.  
\begin{assumption}
\label{ass:Sc-Sc} The function $f(x,y)$ is $\mu$-strongly convex-$\mu$-strongly concave (SC-SC) over $\mathcal{X} \times \mathcal{Y}$, i.e., for all $x_1,x_2\in \mathcal{X}$ and $y_1,y_2\in \mathcal{Y}$, the following holds:
\begin{equation}
\resizebox{1\hsize}{!}{$
\begin{aligned}
    f(x_2,y_1) &\geq f(x_1,y_1) +\langle \nabla_{x} f(x_1,y_1) , x_2-x_1 \rangle  + \frac{\mu}{2} \|x_2-x_1\|^2, \\
    f(x_1,y_2)& \leq f(x_1,y_1) +\langle \nabla_{y} f(x_1,y_1) , y_2-y_1 \rangle  - \frac{\mu}{2} \|y_2-y_1\|^2.
\end{aligned}
$}
\nonumber
\end{equation}
 \end{assumption}

For $x\in \mathcal{X}$ and $y\in\mathcal{Y}$, define $z \triangleq [x;y]$. We have the following result for functions satisfying Assumption \ref{ass:Sc-Sc}. 
 

\begin{theorem}
\label{thm:SC} Suppose Assumptions \ref{ass:smoothness} and \ref{ass:Sc-Sc} hold in conjunction with the assumptions on $\alpha$ and $M$ in Theorem~\ref{thm:CC}. Then, with $\delta=1/(4dT^2)$ and step-size $\eta \leq 1/(4L)$, \texttt{RDEG} guarantees the following with probability at least $1-1/T$:
\begin{equation}
     \| z^{*}-z_{T+1} \|^2\leq  2e^{-\frac{T}{4\kappa}}D^2 +\tilde{O}\left(\frac{\sigma D \kappa}{L}  \left(\sqrt{\alpha}+\sqrt{\frac{1}{M}}\right)\right),
\label{eqn:SC_final_bnd}
\end{equation}
where $\kappa=\mu/L$. 
\end{theorem}

Theorem \ref{thm:SC} says that for smooth strongly convex-strongly concave functions, the iterates generated by \texttt{RDEG} converge linearly to a ball around the saddle point $(x^*,y^*)$ with high probability. The size of the ball is dictated by the second term in Eq.~\eqref{eqn:SC_final_bnd}. 

\begin{remark} (Comments on $\alpha$) The requirement that the fraction of corruption $\alpha \in [0,1/16)$ in our results is inherited from the analysis of the trimmed mean estimator in \cite{lugosi}. One can potentially tolerate a larger fraction of corruption (up to $\alpha < 1/2$) by using the robust estimators in \cite{yin}. However, this would likely come at a price: the authors in \cite{yin} impose additional statistical assumptions on the partial gradients; we do not make such assumptions. 

\begin{remark} (Comments on $M$) In our results, we need the number of agents $M$ to scale with $\log(d T)$. We note that similar conditions show up in the context of adversarially-robust distributed statistical learning; see, for instance, \cite{yin} and \cite{pillutla}. In fact, the covering argument in \cite{yin} requires $M$ to scale linearly with the model dimension $d$. By avoiding such an argument in our analysis, we can get by with a far milder logarithmic dependence on $d$. As an example, for $d=100$, and number of iterations $T=2^{10}$ (which should suffice for all practical purposes), $\log(dT) \approx 12$. This is a very reasonable requirement for large-scale computing systems where the number of devices is of the order of thousands. Furthermore, with $T=2^{10}$, our guarantees in Theorems \ref{thm:CC} and \ref{thm:SC} hold with probability $1-1/T \approx 1$.   
\end{remark}

\end{remark}
\section{Proofs of the Main Results}
\label{sec:Proofs}
In this section, we prove our main results, starting with Theorem \ref{thm:CC}. Essentially, our proofs comprise of a perturbation analysis of the extra-gradient algorithm, where the perturbations arise due to adversarial corruption. As the starting point of such an analysis, we establish some simple relations in the following lemma. 

\begin{lemma}
\label{lemma:basic}
For the \texttt{RDEG} algorithm, the following inequalities hold for all $t\in [T], x\in \mathcal{X}$, and $y\in\mathcal{Y}$:
\begin{equation}
\resizebox{1\hsize}{!}{$
    \begin{aligned}
    2\eta \langle \tilde{g}_x(x_t,y_t), \hat{x}_t-x \rangle &\leq {\Vert x-x_t \Vert}^2 - {\Vert x-\hat{x}_t \Vert}^2 - {\Vert \hat{x}_t-x_t \Vert}^2 \\ 
    -2\eta \langle \tilde{g}_y(x_t,y_t), \hat{y}_t-y \rangle &\leq {\Vert y-y_t \Vert}^2 - {\Vert y-\hat{y}_t \Vert}^2 - {\Vert \hat{y}_t-y_t \Vert}^2 \\ 
    2\eta \langle \tilde{g}_x(\hat{x}_t,\hat{y}_t), {x}_{t+1}-x \rangle &\leq {\Vert x-x_t \Vert}^2 - {\Vert x-{x}_{t+1} \Vert}^2 - {\Vert {x}_{t+1}-x_t \Vert}^2 \\ 
   - 2\eta \langle \tilde{g}_y(\hat{x}_t,\hat{y}_t), {y}_{t+1}-y \rangle &\leq {\Vert y-y_t \Vert}^2 - {\Vert y-{y}_{t+1} \Vert}^2 - {\Vert {y}_{t+1}-y_t \Vert}^2.
    \end{aligned}
$}
\label{eqn:basic_relations}
\end{equation}
\end{lemma}
\begin{proof}
We only prove the first inequality since the rest follow a similar reasoning. We start by noting that
$$ \hat{x}_t = \argmin_{x\in\mathcal{X}} {\Vert x- (x_t - \eta \tilde{g}_x(x_t,y_t))\Vert }^2. $$
From the first order condition for optimality of $\hat{x}_t$, we have that for any $x\in\mathcal{X}$:
$$ \langle x - \hat{x}_t, \hat{x}_t - x_t + \eta \tilde{g}_x(x_t,y_t) \rangle \geq 0. $$
Rearranging the above inequality and simplifying, we obtain:
\begin{equation}
\resizebox{1\hsize}{!}{$
\begin{aligned}
\eta \langle \tilde{g}_x(x_t,y_t), \hat{x}_t-x \rangle & \leq \langle x - \hat{x}_t, \hat{x}_t - x_t \rangle \\
&= \langle x - {x}_t, \hat{x}_t - x_t \rangle - {\Vert x_t - \hat{x}_t \Vert}^2 \\
&\overset{(a)}=\frac{1}{2}\left({\Vert x - {x}_t \Vert}^2 + {\Vert \hat{x}_t - {x}_t \Vert}^2 - {\Vert x - \hat{x}_t \Vert}^2 \right)\\
& \hspace{5mm} - {\Vert x_t - \hat{x}_t \Vert}^2\\
& = \frac{1}{2}\left({\Vert x - {x}_t \Vert}^2 - {\Vert {x} - \hat{x}_t \Vert}^2 - {\Vert \hat{x}_t - {x}_t \Vert}^2 \right),
\end{aligned}
$}
\end{equation}
which leads to the desired claim. For (a), we used the elementary identity that for any two vectors $c,d$, it holds that $2\langle c, d \rangle = {\Vert c \Vert}^2 + {\Vert d \Vert}^2 - {\Vert c-d \Vert}^2$.
\end{proof}

Using the previous result, our next goal is to track the progress made by the mid-point vector $(\hat{x}_t,\hat{y}_t)$ in each iteration, as a function of the errors introduced by adversarial corruption. To that end, for each $\bar{x} \in \mathcal{X}$ and $\bar{y} \in \mathcal{Y}$, we define the following errors vectors:
\begin{equation}
\resizebox{1\hsize}{!}{$
    e_x(\bar{x},\bar{y}) \triangleq \tilde{g}_x(\bar{x},\bar{y})-\nabla_x f(\bar{x},\bar{y}); e_y(\bar{x},\bar{y}) \triangleq \tilde{g}_y(\bar{x},\bar{y})-\nabla_y f(\bar{x},\bar{y}).
$}
\label{eqn:errors}
\end{equation}
We have the following key lemma.
\begin{lemma}
\label{lemma:midpoints} Suppose Assumption \ref{ass:smoothness} holds and $\eta \leq 1/(2L)$. For the \texttt{RDEG} algorithm, the following then holds for all $t\in [T], x\in \mathcal{X}$, and $y\in \mathcal{Y}$:
\begin{equation}
\resizebox{1\hsize}{!}{$
    \begin{aligned}
    & \eta \langle \nabla_x f(\hat{x}_t, \hat{y}_t), \hat{x}_t - x \rangle - \eta \langle \nabla_y f(\hat{x}_t, \hat{y}_t), \hat{y}_t - y \rangle \\
    & \leq \frac{1}{2} \left( {\Vert x-x_t \Vert}^2 - {\Vert x-x_{t+1} \Vert}^2 + {\Vert y-y_t \Vert}^2 - {\Vert y-y_{t+1} \Vert}^2\right)\\
    & + \eta D \left(\Vert e_x(x_t, y_t) \Vert + \Vert e_x(\hat{x}_t, \hat{y}_t) \Vert + \Vert e_y(x_t, y_t) \Vert + \Vert e_y(\hat{x}_t, \hat{y}_t) \Vert \right).
    \end{aligned}
$}
\label{eqn:mid_point_bnd}
\end{equation}
\end{lemma}
\begin{proof}
Using the definition of the error vector $e_x(\hat{x}_t,\hat{y}_t)$, we start by observing that
\begin{equation}
\begin{split}
   \eta \langle \nabla_x f(\hat{x}_t, \hat{y}_t), \hat{x}_t - x \rangle = \underbrace{\eta \langle \nabla_x f(\hat{x}_t, \hat{y}_t), \hat{x}_t - x_{t+1} \rangle}_{T_1} \\ + \underbrace{\eta \langle \tilde{g}_x(\hat{x}_t, \hat{y}_t), {x}_{t+1} - x \rangle}_{T_2} 
   + \eta \langle e_x(\hat{x}_t,\hat{y}_t), x-x_{t+1} \rangle.
\end{split}
\label{eqn:mid_bnd1}
\end{equation}
To bound $T_1$, we note that
\begin{equation}
    \begin{split}
        T_1  = \underbrace{\eta \langle \nabla_x f(\hat{x}_t, \hat{y}_t) - \nabla_x f({x}_t, {y}_t), \hat{x}_t - x_{t+1} \rangle}_{T_3} \\ + \underbrace{\eta \langle \tilde{g}_x({x}_t, {y}_t), \hat{x} - x_{t+1}  \rangle}_{T_4} 
        - \eta \langle e_x(\hat{x}_t,\hat{y}_t), \hat{x}_t-x_{t+1} \rangle.
    \end{split}
\label{eqn:mid_bnd2}
\end{equation}
Now using the third inequality in Eq. \eqref{eqn:basic_relations} of Lemma \ref{lemma:basic} to bound $T_2$, and the first inequality in Eq. \eqref{eqn:basic_relations} with $x=x_{t+1}$ to bound $T_4$, we obtain 
\begin{equation}
\resizebox{1\hsize}{!}{$
    T_2 + T_4 \leq \frac{1}{2} \left( {\Vert x- x_t \Vert}^2 - {\Vert x- x_{t+1} \Vert}^2 - {\Vert \hat{x}- x_t \Vert}^2 - {\Vert \hat{x}_t- x_{t+1}  \Vert}^2 \right). 
$} 
\label{eqn:mid_bnd3}
\end{equation}
Recalling that $D=\max\{D_x, D_y\}$ (where $D_x$ and $D_y$ are the diameters of $\mathcal{X}$ and $\mathcal{Y}$, respectively), and combining equations \eqref{eqn:mid_bnd1}, \eqref{eqn:mid_bnd2}, and \eqref{eqn:mid_bnd3}, we conclude that
\begin{equation}
\resizebox{1\hsize}{!}{$
\begin{aligned}
   \eta \langle \nabla_x f(\hat{x}_t, \hat{y}_t), \hat{x}_t - x \rangle & \leq \Psi_{x,t} + \frac{1}{2} \left({\Vert x- x_t \Vert}^2 - {\Vert x- x_{t+1} \Vert}^2\right) \\
   & + \eta D \left(\Vert e_x(x_t, y_t) \Vert + \Vert e_x(\hat{x}_t, \hat{y}_t) \Vert \right),
\end{aligned}
$}
\label{eqn:mid_bnd4}
\end{equation}
where 
$$ \Psi_{x,t} = T_3 - \frac{1}{2}\left({\Vert \hat{x}- x_t \Vert}^2 + {\Vert \hat{x}_t- x_{t+1}  \Vert}^2 \right). $$ 
Using a similar string of arguments, we can establish that 
\begin{equation}
\resizebox{1\hsize}{!}{$
\begin{aligned}
   - \eta \langle \nabla_y f(\hat{x}_t, \hat{y}_t), \hat{y}_t - y \rangle & \leq \Psi_{y,t} + \frac{1}{2} \left({\Vert y- y_t \Vert}^2 - {\Vert y- y_{t+1} \Vert}^2\right) \\
   & + \eta D \left(\Vert e_y(x_t, y_t) \Vert + \Vert e_y(\hat{x}_t, \hat{y}_t) \Vert \right),
\end{aligned}
$}
\label{eqn:mid_bnd5}
\end{equation}
where 
$$ \Psi_{y,t} = T_5 - \frac{1}{2}\left({\Vert \hat{y}- y_t \Vert}^2 + {\Vert \hat{y}_t- y_{t+1}  \Vert}^2 \right), \text{and} $$ 
$$ T_5 = - \eta \langle \nabla_y f(\hat{x}_t, \hat{y}_t) - \nabla_y f({x}_t, {y}_t), \hat{y}_t - y_{t+1} \rangle. $$ 
To complete the proof, we claim that $\Psi_{x,t} + \Psi_{y,t} \leq 0$. To see why this is the case, we note that $L$-smoothness yields:
\begin{equation}
\begin{aligned}
    \Vert \nabla_x f(\hat{x}_t,\hat{y}_t) - \nabla_x f(x_t,y_t) \Vert &\leq L \left(\Vert \hat{x}_t-x_t \Vert + \Vert \hat{y}_t - y_t \Vert \right), \\
    \Vert \nabla_y f(\hat{x}_t,\hat{y}_t) - \nabla_y f(x_t,y_t) \Vert &\leq L \left(\Vert \hat{x}_t-x_t \Vert + \Vert \hat{y}_t - y_t \Vert \right).
\end{aligned}   
\nonumber
\end{equation}
Using the above display in conjunction with the  Cauchy-Schwarz inequality, we obtain
\begin{equation}
\resizebox{1\hsize}{!}{$
\begin{split}
\Psi_{x,t}+\Psi_{y,t} \leq L\eta \left(\Vert \hat{x}_t-x_t \Vert + \Vert \hat{y}_t - y_t \Vert \right) \left(\Vert \hat{x}_t-x_{t+1} \Vert + \Vert \hat{y}_t - y_{t+1} \Vert \right) \\ 
- \frac{1}{2}\left({\Vert \hat{x}- x_t \Vert}^2 + {\Vert \hat{x}_t- x_{t+1}  \Vert}^2 + {\Vert \hat{y}- y_t \Vert}^2 + {\Vert \hat{y}_t- y_{t+1}  \Vert}^2 \right).
\end{split}
$}
\nonumber
\end{equation}
Finally, using $\eta \leq 1/(2L)$ along with the fact that for any $a,b,c,d \in \mathbb{R}$, $(a+b)(c+d) \leq a^2 + b^2 +c^2 + d^2$, we conclude that $\Psi_{x,t}+\Psi_{y,t} \leq 0$. The claim in Eq. \eqref{eqn:mid_point_bnd} then follows from summing equations \eqref{eqn:mid_bnd4} and \eqref{eqn:mid_bnd5}.
\end{proof}

In our next result, we establish high-probability bounds on the error vectors by leveraging Theorem \ref{thm:lugosi}. 
\begin{lemma} \label{lemma:high_prob} 
Consider the event $\mathcal{H}_t$ defined as follows:
\begin{equation}
\resizebox{1\hsize}{!}{$
 \mathcal{H}_t \triangleq \{\max\{\Vert e_x(x_t,y_t) \Vert, \Vert e_x(\hat{x}_t,\hat{y}_t) \Vert, \Vert e_y(x_t,y_t) \Vert, \Vert e_y(\hat{x}_t,\hat{y}_t) \Vert\} \leq \Delta\}, 
 $}
 \nonumber
\end{equation}
where 
\begin{equation}
\Delta = c\sigma \left(\sqrt{\alpha} + \sqrt{\frac{\log(4dT^2)}{M}} \right).
\label{eqn:Delta}
\end{equation}
For the \texttt{RDEG} algorithm, we have:
\begin{equation}
    \mathbb{P}\left( \mathcal{H}_t \right) 
    \geq 1-\frac{1}{T^2}, \hspace{1mm} \text{for each $t \in [T]$}. 
\label{eqn:prob_bnd}
\end{equation}
\end{lemma}
\begin{proof}
We begin by defining certain ``good" events:
\begin{equation}
    \begin{aligned}
    \mathcal{G}_{x,t} \triangleq \{ \Vert e_x(x_t,y_t) \Vert \leq \Delta \}, \hspace{1mm} &  \mathcal{G}_{y,t} \triangleq \{ \Vert e_y(x_t,y_t) \Vert \leq \Delta \}, \\
    \mathcal{\bar{G}}_{x,t} \triangleq \{ \Vert e_x(\hat{x}_t,\hat{y}_t) \Vert \leq \Delta \}, \hspace{1mm} &  \mathcal{\bar{G}}_{y,t} \triangleq \{ \Vert e_y(\hat{x}_t,\hat{y}_t) \Vert \leq \Delta \}. 
    \end{aligned}
\nonumber
\end{equation}
To analyze the probability of occurrence of the above events, we need to next define an appropriate filtration. Accordingly, let $\mathcal{F}_t$ denote the sigma field generated by $\{x_k,y_k\}_{k\in[t]}$ and $\{\hat{x}_k, \hat{y}_k\}_{k\in[t-1]}$; and $\mathcal{\bar{F}}_t$ denote the sigma field generated by $\{x_k,y_k\}_{k\in[t]}$ and $\{\hat{x}_k, \hat{y}_k\}_{k\in[t]}$. From  definition, we have
$$ \mathcal{F}_1 \subset \mathcal{\bar{F}}_1  \subset \mathcal{F}_2 \subset \mathcal{\bar{F}}_2 \subset \cdots \subset  \mathcal{F}_T \subset \mathcal{\bar{F}}_T. $$
Clearly, $(x_t,y_t)$ is $\mathcal{F}_t$-measurable. Thus, conditioned on $\mathcal{F}_t$, for each coordinate $j\in [n]$, the data set $\{[g_x({x}_t,{y}_t; \xi^{(i)}_{1,t})]_j : i \in [M]\}$ has the following properties: (i) at most $\alpha M$ of the samples are corrupted; and (ii) the uncorrupted samples are i.i.d. scalar random variables with mean $[\nabla_x f(x_t,y_t)]_j$ and variance bounded above by $\sigma^2_x(j)$. Invoking Theorem \ref{thm:lugosi} for the trimmed mean estimator in Algorithm \ref{alg:trim}, we conclude that conditioned on $\mathcal{F}_t$, with probability at least $1-1/(4dT^2)$,
\begin{equation}
    \resizebox{1\hsize}{!}{$
 \vert [\tilde{g}_x(x_t,y_t)]_j - [\nabla_x f(x_t,y_t)]_j \vert \leq c\sigma_x(j) \left(\sqrt{\alpha} + \sqrt{\frac{\log(4dT^2)}{M}} \right).
$}
\nonumber
\end{equation}
Now union-bounding over each of the $n$ coordinates, we have that conditioned on $\mathcal{F}_t$, with probability at least $1-\frac{n}{4dT^2} \geq 1-\frac{1}{4T^2}$, 
$$ \Vert \tilde{g}_x(x_t,y_t) - \nabla_x f(x_t,y_t) \Vert \hspace{0.5mm} \leq \hspace{0.5mm} \Delta. $$
Here, we used the fact that $d=\max\{n,m\} \geq n$, and $\sqrt{\sum_{j\in[n]} \sigma^2_x(j)} = \sigma_x \leq \sigma$. We have thus shown that $\mathbb{P}\left(\mathcal{G}_{x,t}|\mathcal{F}_t\right) \geq 1-1/(4T^2).$ Using an identical argument, we can establish an analogous result for the event $\mathcal{G}_{y,t}$. An union bound thus yields $\mathbb{P}\left(\mathcal{G}_{t}|\mathcal{F}_t\right) \geq 1-1/(2T^2)$, where $\mathcal{G}_t=\mathcal{G}_{x,t} \cap \mathcal{G}_{y,t}$. Noting that $(\hat{x}_t,\hat{y}_t)$ is $\mathcal{\bar{F}}_t$-measurable, we can similarly show that $\mathbb{P}\left(\mathcal{\bar{G}}_{t}|\mathcal{\bar{F}}_{t}\right) \geq 1-1/(2T^2)$, where $\mathcal{\bar{G}}_t=\mathcal{\bar{G}}_{x,t} \cap \mathcal{\bar{G}}_{y,t}$. Our next task is to analyze the probability of occurrence of the event $\mathcal{H}_t = \mathcal{{G}}_t \cap \mathcal{\bar{G}}_t$ by exploiting the nested sigma-field structure: $\mathcal{F}_t \subset \mathcal{\bar{F}}_t.$ To that end, observe:
\begin{equation}
    \begin{aligned}
    \mathbb{P}(\mathcal{\bar{G}}_t|\mathcal{F}_t) &= \mathbb{E}[{1}_{\mathcal{\bar{G}}_t}|\mathcal{F}_t] \\
    &\overset{(a)}=\mathbb{E}[\mathbb{E}[{1}_{\mathcal{\bar{G}}_t}|\mathcal{\bar{F}}_t]|\mathcal{F}_t]\\
    &=\mathbb{E}[\mathbb{P}(\mathcal{\bar{G}}_t|\mathcal{\bar{F}}_t)|\mathcal{F}_t]\\
    &\overset{(b)}\geq 1-\frac{1}{2T^2}.
    \end{aligned}
\label{eqn:prob_nest}
\end{equation}
Here, we used $1_{\mathcal{A}}$ to represent the indicator random variable for an event $\mathcal{A}$. For (a), we used the fact that given a random variable $X$ and two sigma-fields $\mathcal{B}_1$ and $\mathcal{B}_2$ with $\mathcal{B}_1 \subset \mathcal{B}_2$, it holds that $\mathbb{E}[\mathbb{E}[X|\mathcal{B}_2]|\mathcal{B}_1] = \mathbb{E}[X|\mathcal{B}_1]$, i.e., the smaller sigma-field ``wins" \cite[Theorem 5.1.6]{durrett}. For (b), we used the previously established fact that $\mathbb{P}\left(\mathcal{\bar{G}}_{t}|\mathcal{\bar{F}}_{t}\right) \geq 1-1/(2T^2)$. Using \eqref{eqn:prob_nest} and an union bound, we conclude that $\mathbb{P}(\mathcal{H}_t|\mathcal{F}_t) = \mathbb{P}(\mathcal{G}_t \cap \mathcal{\bar{G}}_t |\mathcal{F}_t) \geq 1-1/T^2$. To complete the proof, we note that
$$ \mathbb{P}(\mathcal{H}_t) = \mathbb{E}[1_{\mathcal{H}_t}] = \mathbb{E}[\mathbb{E}[1_{\mathcal{H}_t}|\mathcal{F}_t]] = \mathbb{E}[\mathbb{P}(\mathcal{H}_t|\mathcal{F}_t)] \geq 1 - \frac{1}{T^2}. $$
\end{proof}

We are now equipped with all the pieces needed to prove Theorem \ref{thm:CC}. 

\begin{proof} \textbf{(Theorem \ref{thm:CC})} Let us start by considering the following ``clean" event: $\mathcal{H}= \bigcap_{t\in[T]} \mathcal{H}_t$, where $\mathcal{H}_t$ is as defined in Lemma \ref{lemma:high_prob}. We will condition on this event for the rest of the proof. From the convex-concave property of $f(x,y)$, the following inequalities hold for all  $t\in[T], x\in \mathcal{X}$, and $y\in\mathcal{Y}$:
\begin{equation}
    \begin{aligned}
    \eta \left( f(\hat{x}_t,\hat{y}_t)-f(x,\hat{y}_t) \right) & \leq \eta \langle \nabla_x f(\hat{x}_t, \hat{y}_t), \hat{x}_t - x \rangle\\
   - \eta \left( f(\hat{x}_t,\hat{y}_t)-f(\hat{x}_t,y) \right) & \leq - \eta \langle \nabla_y f(\hat{x}_t, \hat{y}_t), \hat{y}_t - y \rangle. 
    \end{aligned}
\nonumber
\end{equation}
Summing the two inequalities above, using Eq. \eqref{eqn:mid_point_bnd} from Lemma \ref{lemma:midpoints}, and the fact that on the event $\mathcal{H}$, all the error vectors are uniformly bounded above by $\Delta$, we obtain: 
\begin{equation}
    \begin{split}
\eta \left( f(\hat{x}_t,{y})-f(x,\hat{y}_t) \right) \leq \frac{1}{2} \left( {\Vert x-x_t \Vert}^2 - {\Vert x-x_{t+1} \Vert}^2 \right) 
\\ + \frac{1}{2} \left( {\Vert y-y_t \Vert}^2 - {\Vert y-y_{t+1} \Vert}^2\right) + 4\eta D \Delta.
    \end{split}
\end{equation}
Summing the above inequality from $t=1$ to $T$, and simplifying, we obtain
\begin{equation}
\resizebox{1\hsize}{!}{$
    \begin{aligned}
\frac{1}{T} \sum_{t\in [T]} \left(f(\hat{x}_t,{y})-f(x,\hat{y}_t) \right) & \leq \frac{1}{2 \eta T} \left( {\Vert x-x_1 \Vert}^2 + {\Vert y-y_{1} \Vert}^2 \right) \\
& \hspace{1mm} + 4D\Delta. 
    \end{aligned}
$}
\label{eqn:final_bnd1}
\end{equation}
Now from the convexity of $f(x,y)$ w.r.t. $x$ and concavity w.r.t. $y$, we have that $f(\bar{x}_T,y) \leq (1/T) \sum_{t\in[T]} f(\hat{x}_t,y)$ and  $f(x,\bar{y}_T) \geq  (1/T) \sum_{t\in[T]} f(x,\hat{y}_t)$, respectively. Using these facts in conjunction with Eq. \eqref{eqn:final_bnd1}, we obtain
\begin{equation}
f(\bar{x}_T,y)-f(x, \bar{y}_T) \leq \frac{D^2}{\eta T} + 4 D \Delta.
\label{eqn:duality}
\end{equation}
Noting that the above inequality holds for all $x\in\mathcal{X}$ and for all $y\in \mathcal{Y}$ immediately leads to the claim in Eq.~\eqref{eqn:CC_dual_bnd}.
To complete the proof, we observe that
$$ \mathbb{P}(\mathcal{H}^c) \leq \sum_{t\in[T]} \mathbb{P}(\mathcal{H}^c_t) \leq T \times \frac{1}{T^2} = \frac{1}{T}, $$
where we used the union bound, and Lemma \ref{lemma:high_prob}. 
\end{proof}

(\textbf{Proof of Corollary \ref{corr:avg_it}}) 
Starting from Eq.~\eqref{eqn:duality} and plugging in $x=x^*, y=\bar{y}_T$, we obtain
$$ f(\bar{x}_T,\bar{y}_T)-f(x^*, \bar{y}_T) \leq \frac{D^2}{\eta T} + 4 D \Delta. $$
Given that the saddle point property of $(x^*,y^*)$ implies $f(x^*,y^*) \geq f(x^*,\bar{y}_T)$, we conclude that $f(\bar{x}_T,\bar{y}_T) - f(x^*,y^*) \leq D^2/(\eta T) + 4 D \Delta$. Using a symmetric argument, we can arrive at the conclusion that on the clean event $\mathcal{H}$,
$$ |f(\bar{x}_T,\bar{y}_T)-f(x^*, y^*)| \leq \frac{D^2}{\eta T} + 4 D \Delta. $$ 

We now turn to the proof of Theorem \ref{thm:SC}. Let us first introduce some notation to simplify the exposition. For $x\in \mathcal{X}$ and $y\in\mathcal{Y}$, define $\mathcal{Z} \triangleq \mathcal{X} \times \mathcal{Y}$,  $z \triangleq [x;y]$, $F(z) \triangleq [\nabla_{x}f(x,y);-\nabla_{y}f(x,y)]$, and $\tilde G(z) \triangleq [\tilde g_x(x,y);-\tilde g_y(x,y)]$. We will require the following intermediate lemma. 
\begin{lemma}\label{lemma:Sc-Sc}\cite[Lemma 2]{zhou2018fenchel} If $f$ is $\mu$-strongly convex-$\mu$-strongly concave, then the following holds for all $z_1,z_2\in \mathcal{Z}$:
\begin{equation}
    \langle F(z_2)-F(z_1) , z_2-z_1 \rangle \geq \mu \|z_2-z_1\|^2.
\end{equation}
\end{lemma}

\begin{proof} \textbf{(Theorem~\ref{thm:SC})}
Similar to the proof of the Theorem~\ref{thm:CC}, we will condition on the clean event $\mathcal{H}= \bigcap_{t\in[T]} \mathcal{H}_t$, where $\mathcal{H}_t$ is as defined in Lemma \ref{lemma:high_prob}. From Lemma~\ref{lemma:basic}, we have
 \begin{equation}
\resizebox{1\hsize}{!}{$
    \begin{aligned}
    2\eta\langle  \tilde G(\hat z_t),z_{t+1}-z\rangle\leq \|z-{z}_t\|^2-\|z-{z}_{t+1}\|^2-\|{z}_{t+1}- z_t\|^2\\
     2\eta\langle \tilde G(z_t),\hat{z}_{t}-z_{t+1}\rangle\leq \|z_{t+1}-{z}_t\|^2-\|\hat z_t-{z}_{t+1}\|^2-\|{z}_{t}-\hat z_t\|^2.
    \end{aligned}
$}
\label{eq:sc eg 2 Faulty}
\end{equation}
If we let $z= z^{*}$ in the first inequality above, we obtain:
 \begin{equation}
 \resizebox{1\hsize}{!}{$
     2\eta\langle  \tilde G(\hat z_t), z_{t+1}- z^{*}\rangle\leq \| z^{*}-{ z}_t\|^2-\| z^{*}-{ z}_{t+1}\|^2-\|{ z}_{t+1}-  z_t\|^2.
     $}
 \label{eq:sc eg 3 Faulty}
 \end{equation}
From Lemma~\ref{lemma:high_prob}, we also know that on the clean event $\mathcal{H}$, the following holds:
\begin{equation}
      \| F(\hat z_t)- \tilde G(\hat z_t)\|\leq \sqrt{2}\Delta,\quad
      \| F( z_t)- \tilde G( z_t)\|\leq\sqrt{2} \Delta,
\label{eq:Faulty_bnd_1}
 \end{equation}
where $\Delta$ is as in Eq. \eqref{eqn:Delta}. Now using the Cauchy–Schwarz inequality, we obtain  
\begin{equation}
 \begin{aligned}
     2\eta\langle   F(\hat z_t),z_{t+1}-z^{*}\rangle-2\eta\langle  \tilde G(\hat z_t),z_{t+1}-z^{*}\rangle\leq 4\eta \Delta D,\\
     2\eta\langle  F(z_t),\hat{z}_{t}-z_{t+1}\rangle-2\eta\langle \tilde G(z_t),\hat{z}_{t}-z_{t+1}\rangle \leq 4\eta\Delta D.
 \end{aligned}
 \nonumber
\end{equation}
 Combining the above display with  \eqref{eq:sc eg 2 Faulty}, \eqref{eq:sc eg 3 Faulty}, and \eqref{eq:Faulty_bnd_1} yields:
 \begin{equation}
  \label{eq:sc eg 4 Faulty}
\resizebox{1\hsize}{!}{$
    \begin{aligned}
     &2\eta\langle F(z_t),\hat{z}_{t}-z_{t+1}\rangle+2\eta\langle F(\hat z_t),z_{t+1}-z^{*}\rangle \\&\leq \|z^{*}-{z}_t\|^2-\|z^{*}-{z}_{t+1}\|^2-\|\hat z_t-{z}_{t+1}\|^2-\|{z}_{t}-\hat z_t\|^2+8\eta\Delta D.
     \end{aligned}
$}
\end{equation}
Our immediate goal is to obtain a lower bound for the LHS of the above inequality. 
To that end, we start by using the fact that if $z^* \in \mathcal{Z}$ is a saddle point of $f(x,y)$, then  $\forall z\in \mathcal{Z}$, we have (see Section 2.3  in \cite{gidel2018variational}): 
\begin{equation}
    \langle F(z^*),z-z^*\rangle \geq 0.
\end{equation}
 This readily implies that
 \begin{equation}
  \label{eq:sc eg 4.25 Faulty}
\resizebox{1\hsize}{!}{$
    \begin{aligned}
    2\eta&\langle  F( z_{t}),\hat z_{t}- z_{t+1}\rangle +2\eta\langle F(\hat z_t), z_{t+1}- z^{*}\rangle\\= &2\eta\langle  F( z_{t})-F(\hat z_t),\hat z_t- z_{t+1}\rangle +2\eta\langle F(\hat z_t),\hat z_t- z^{*}\rangle
    \\ \geq &2\eta\langle  F( z_{t})-F(\hat z_t),\hat z_t- z_{t+1}\rangle +2\eta\langle F(\hat z_t)-F(z^*),\hat z_t- z^{*}\rangle.
     \end{aligned}
$}
\end{equation}
We can further lower-bound the RHS of the above inequality as follows:
  \begin{equation}
 \resizebox{1\hsize}{!}{$
  \begin{aligned}
     &2\eta\langle  F( z_{t})-F(\hat z_t),\hat z_t- z_{t+1}\rangle +2\eta\langle F(\hat z_t)-F(z^*),\hat z_t- z^{*}\rangle
     \\\overset{(a)}\geq& 2\eta\langle  F( z_{t})-F(\hat z_t),\hat z_t- z_{t+1}\rangle+2\eta\mu\|\hat z_t- z^{*}\|^2
     \\\overset{(b)}\geq& -4\eta L\| z_{t}-\hat z_t\| \|\hat z_t- z_{t+1}\|+2\eta\mu\|\hat z_t- z^{*}\|^2
     \\\overset{(c)}\geq&-(4\eta^2 L^2\| z_{t}-\hat z_t\|^2 +\|\hat z_t- z_{t+1}\|^2)+2\eta\mu\|\hat z_t- z^{*}\|^2,
 \end{aligned}
 $}
\label{eqn:Thm3lbound}
\end{equation}
where we used Lemma~\ref{lemma:Sc-Sc} for (a);  the smoothness property of $f$ in  Assumption~\ref{ass:smoothness} for (b); and the AM-GM inequality for (c). Combining equations \eqref{eq:sc eg 4 Faulty}, \eqref{eq:sc eg 4.25 Faulty}, and \eqref{eqn:Thm3lbound}, we obtain 
\begin{equation}
 \label{eq:scsc proof comb}
\resizebox{1\hsize}{!}{$
 \begin{aligned}
     &\|  z^{*}-z_{t+1}\|^2
     \\&\leq
     \| z^{*}-z_{t}\|^2-2 \eta\mu\|\hat z_t- z^{*}\|^2+( 4\eta^2L^2-1)\| z_{t}-\hat z_t\|^2+8 \eta\Delta D
     \\&=(1- \eta\mu)\|  z^{*}-z_{t}\|^2+(-1+ 4\eta^2L^2+2 \eta\mu)\| z_{t}-\hat z_t\|^2
     \\&\quad+ \eta\mu\| z^{*}-z_{t}\|^2-2 \eta\mu\|\hat z_t- z^{*}\|^2-2 \eta\mu\| z_{t}-\hat z_t\|^2+8 \eta\Delta D.
 \end{aligned}
$}
\end{equation}
To simplify the above display, we use 
the elementary fact that for any two vectors $a$ and $b$, it holds that ${\Vert a + b \Vert}^2 \leq 2 \left(\Vert a \Vert^2 + \Vert b \Vert^2\right)$; this yields:
\begin{align}
 \eta\mu\| z_{t}- z^{*}\|^2-2 \eta\mu\|\hat z_t- z^{*}\|^2-2 \eta\mu\| z_{t}-\hat z_t\|^2\leq0.
\label{eqn:simp1}
\end{align}
Next, setting $\eta=\frac{1}{4L}$, and using $\mu\leq L$, we have:
\begin{align}
    -1+ 4\eta^2L^2+2 \eta\mu\leq-1+ 4\eta^2L^2+2 \eta L\leq0.
\label{eqn:simp2}
\end{align}
Using \eqref{eqn:simp1} and \eqref{eqn:simp2} to simplify \eqref{eq:scsc proof comb}, we finally obtain:
  \begin{align}
     \| z^{*}-z_{t+1} \|^2&\leq
     (1- \eta\mu)\|z^{*}- z_{t}\|^2+8 \eta\Delta D\notag\\
     &=\left(1-\frac{1}{4\kappa}\right)\| z^{*}-z_{t}\|^2 +2 \frac{\Delta D}{L},
 \end{align}
where $\kappa=\mu/L$. Using the above inequality recursively yields:
\begin{equation}
\resizebox{1\hsize}{!}{$
\begin{aligned}
     \| z^{*}-z_{T+1} \|^2&\leq
   \left(1-\frac{1}{4\kappa}\right)^T\| z^{*}-z_{1}\|^2 +2\sum_{t=0}^{T-1}\left(1-\frac{1}{4\kappa}\right)^t \frac{\Delta D}{L}
   \notag\\&\leq
   \left(1-\frac{1}{4\kappa}\right)^T\| z^{*}-z_{1}\|^2 +8 \frac{\kappa\Delta D}{L}
   \notag\\
     &\leq e^{-\frac{T}{4\kappa}}\| z^{*}-z_{1}\|^2 +8 \frac{\kappa\Delta D}{L},
 \end{aligned}
 $}
\end{equation}
where for the last step, we used $(1-w)^r \leq e^{-rw}$ for $w\in (0,1)$, and $r \geq 0$. To conclude, we recall from the proof of Theorem \ref{thm:CC} that the event $\mathcal{H}$ has measure at least $1-1/T$. 
 \end{proof} 
 \section{Simulations}
\label{sec:sims}
 In this section, we study a specific instance of problem \eqref{eqn:objective}, namely, a bilinear game of the following form: 
 \begin{equation}
     \min_{\Vert x \Vert \leq \rho} \max_{\Vert y \Vert \leq \rho}f(x,y) \triangleq \mathbb{E}[x^TAy +2(b+\zeta)^Tx-2(c+\zeta)^T y].
\nonumber
 \end{equation}
Here, $x,y,b,c\in \mathbb{R}^{10}$, $A\in \mathbb{R}^{10\times 10}$, and $\rho=100$. The parameters $A,b,c$ are fixed, and  $\zeta\sim N(0,\sigma^2 I)$. As our measure of performance, we consider the instantaneous  primal-dual gap  ${\phi_t=}\max_{y\in \mathcal{Y}} f(\bar{x}_t,y)-\min_{x\in\mathcal{X}}f(x,\bar{y}_t)$. We simulate two algorithms: the vanilla extra-gradient algorithm that does not account for adversaries, and the proposed  \texttt{RDEG} algorithm. In Fig.~\ref{fig:allfig}(a), we plot the performance of these algorithms with the corruption fraction set to $\alpha=0.06$, number of agents $M=100$, and variance of the noise set to $\sigma^2=10$. We observe that even a small number of Byzantine workers can manipulate the optimization procedure and cause the extra-gradient algorithm to diverge from the saddle point. In Fig.~\ref{fig:allfig}(b), with $M=100$ and  $\sigma^2=10$, we explore the impact of varying the corruption fraction $\alpha$. Complying with Theorem \ref{thm:CC}, the error floor of \texttt{RDEG} increases as a function of $\alpha$. Next, in Fig.~\ref{fig:allfig}(c), to demonstrate the benefit of collaboration, 
 we fix $\alpha=0.06$ and $\sigma^2=10$, and plot the performance of \texttt{RDEG} as a function of the number of agents $M$. As expected, by increasing $M$, \texttt{RDEG} converges to a smaller ball around the saddle point, highlighting the benefit of collaboration in reducing the variance of the noise model. Finally, in Fig.~\ref{fig:allfig}(d), we fix $M=100$ and $\alpha=0.06$, and change the variance of the noise $\sigma^2$. We observe that increasing $\sigma^2$ leads to a higher error-floor. 
 Importantly, all of the above plots verify the bound in Theorem~\ref{thm:CC}.
 \begin{figure}[t]

\begin{minipage}{0.245\textwidth}
     \begin{figure}[H]
        \includegraphics[width=\linewidth]{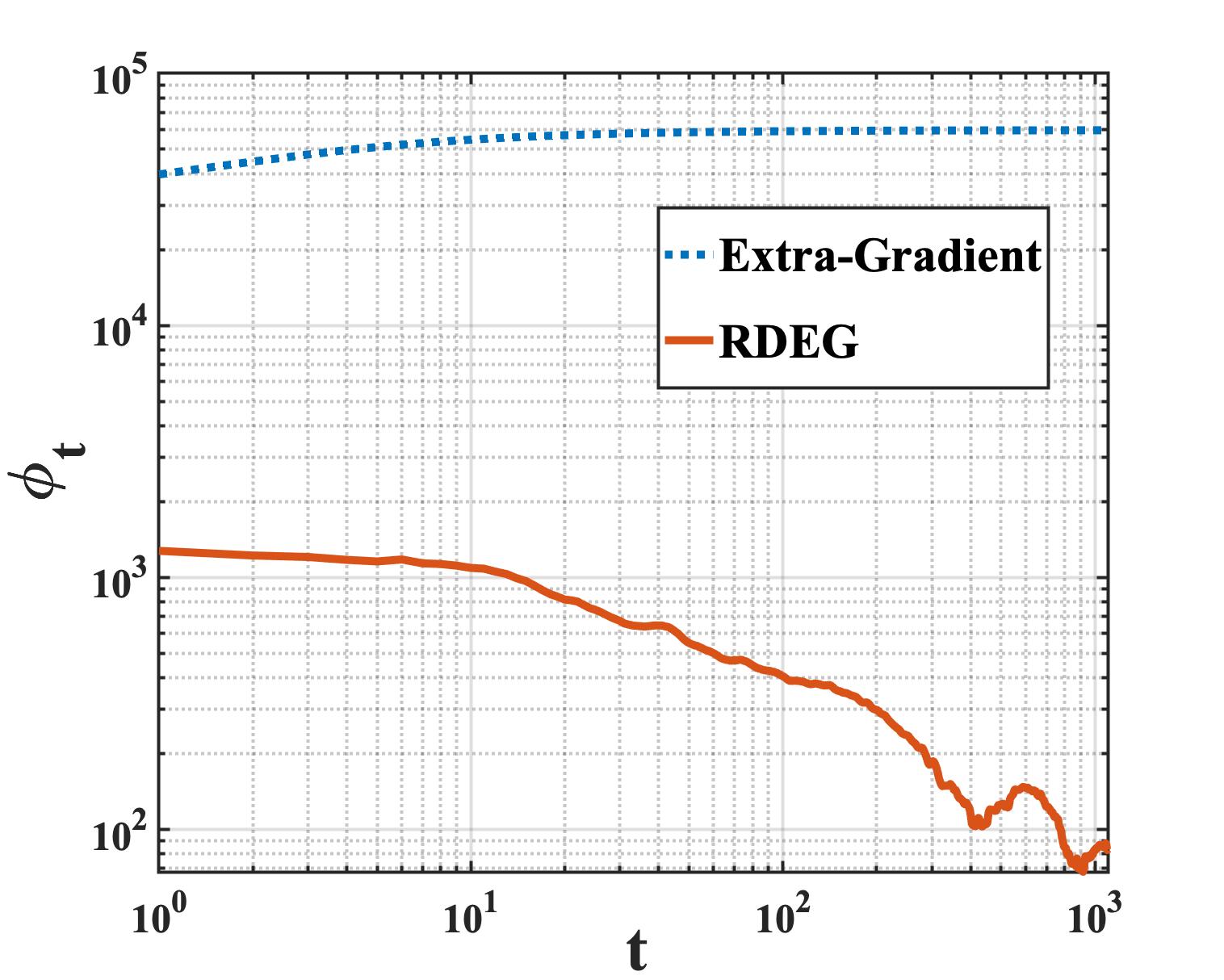}
        \label{fig:syn1}  
    \end{figure}
    \end{minipage}
    \hspace{-4mm}
\begin{minipage}{0.245\textwidth}
    \begin{figure}[H]
        \includegraphics[width=\linewidth]{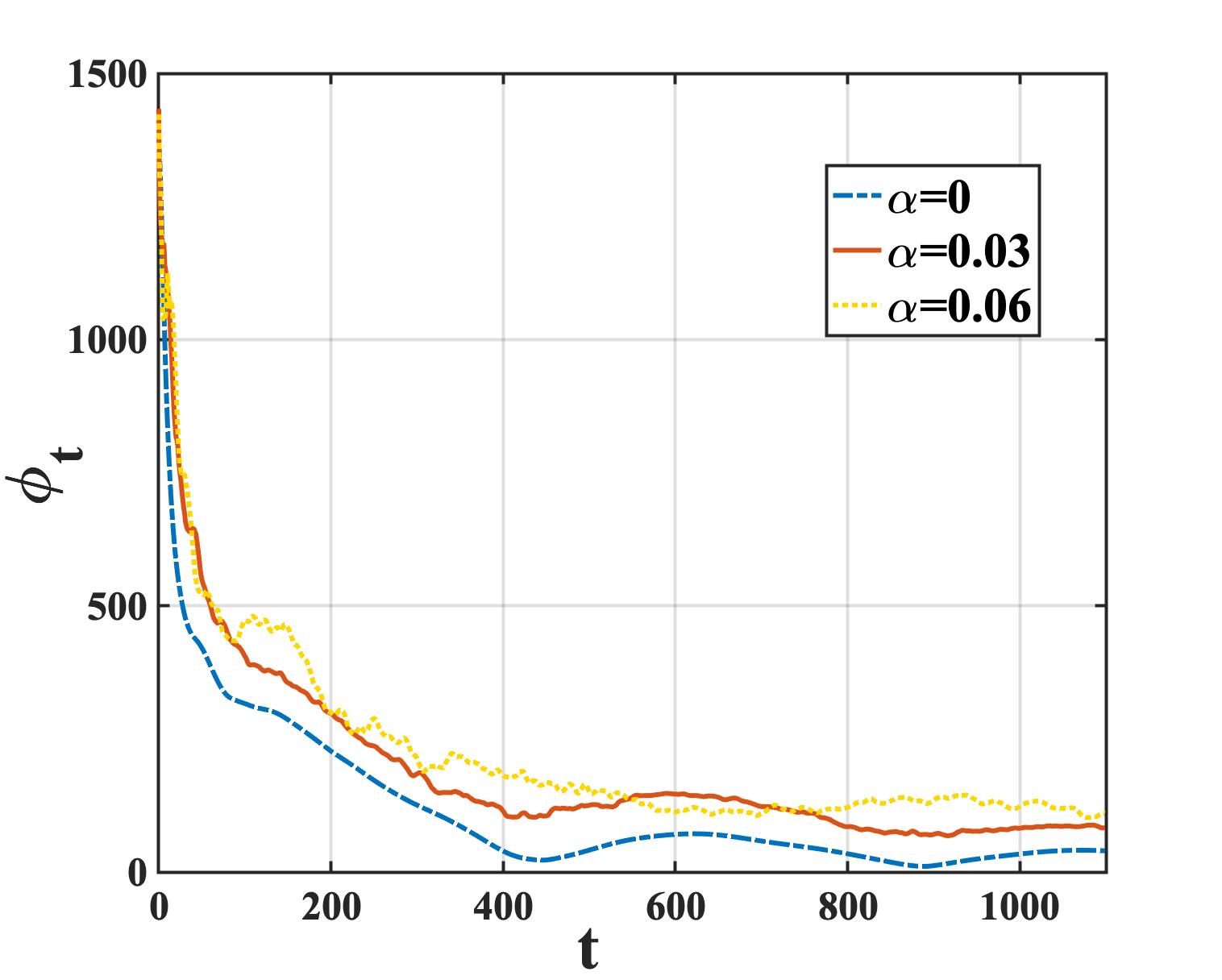}
        \label{fig:synlog2}
    \end{figure} %
    \end{minipage}
    \hspace{-4mm}
    \vspace{-3mm}
    
    \vspace{-5mm}

\begin{minipage}{0.245\textwidth}
     \begin{figure}[H]
        \includegraphics[width=\linewidth]{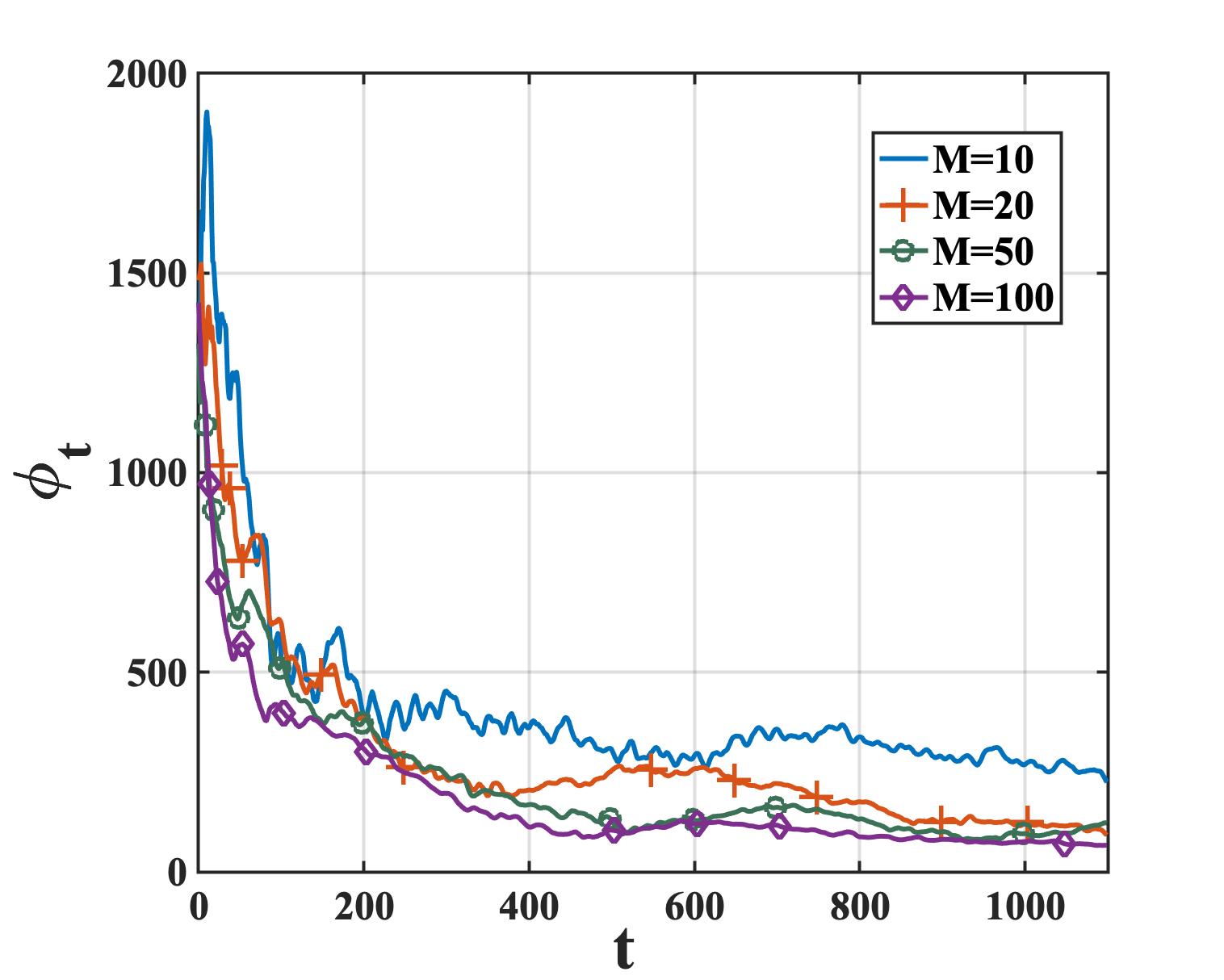}
        \label{fig:syn}  
    \end{figure}
    \end{minipage}
    \hspace{-4mm}
\begin{minipage}{0.245\textwidth}
    \begin{figure}[H]
        \includegraphics[width=\linewidth]{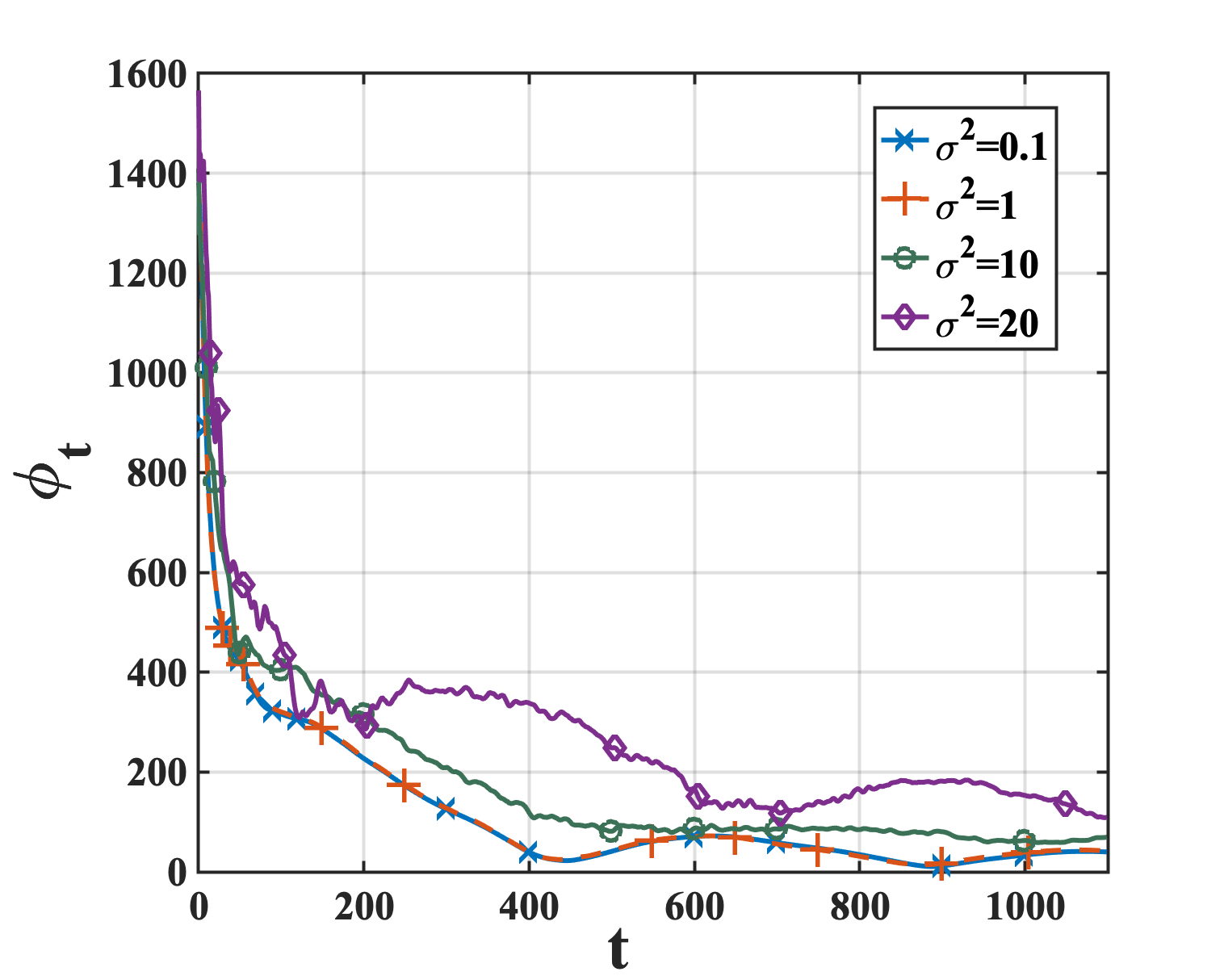}
        \label{fig:synlog}
    \end{figure} %
    \end{minipage}
    \hspace{-4mm}
   
    \vspace{-3mm}
    
    \caption{\textbf{Top Left (a). }Comparison between vanilla  extra-gradient and RDEG. \textbf{Top Right (b).} Performance of RDEG vs. level of corruption fraction. \textbf{Bottom Left  (c).} Performance of RDEG vs. number of agents. \textbf{Bottom Right  (d).} Performance of RDEG vs. level of noise variance.}\label{fig:allfig}
    \vspace{-2mm}
\end{figure}

\section{Conclusion}
We studied the problem of distributed min-max learning under adversarial agents for the first time. By exploiting recent ideas from robust statistics, we developed a novel  robust distributed extra-gradient algorithm. For both smooth convex-concave and smooth strongly convex-strongly concave functions, we showed that with high probability, our proposed approach guarantees convergence to approximate saddle points at near-optimal statistical rates. 

\bibliographystyle{IEEEtran} 
\bibliography{refs}
\end{document}